\documentclass[letterpaper, 10 pt, conference]{ieeeconf}  

\IEEEoverridecommandlockouts                              
\usepackage{times} 

\usepackage[utf8]{inputenc}
\usepackage{hyperref}
\usepackage{booktabs}
\usepackage{multirow}
\usepackage{subcaption} 
\usepackage{array}
\usepackage{xspace}
\usepackage{flushend}
\usepackage{algorithm,algorithmicx}

\usepackage{amsmath,amssymb,amsfonts,amsthm}
\usepackage[noend]{algpseudocode}
\usepackage{comment}
\usepackage{siunitx}
\usepackage{soul}

\hypersetup{urlcolor=blue, colorlinks=true}

\usepackage[colorinlistoftodos,prependcaption,textsize=tiny]{todonotes}

\graphicspath{
    {./figs/}
}

\newcommand{\calM}{\ensuremath{\mathcal{M}}\xspace}
\newcommand{\calW}{\ensuremath{\mathcal{W}}\xspace}
\newcommand{\calP}{\ensuremath{\mathcal{P}}\xspace}
\newcommand{\calE}{\ensuremath{\mathcal{E}}\xspace}
\newcommand{\calX}{\ensuremath{\mathcal{X}}\xspace}
\newcommand{\calG}{\ensuremath{\mathcal{G}}\xspace}
\newcommand{\calO}{\ensuremath{\mathcal{O}}\xspace}
\newcommand{\calD}{\ensuremath{\mathcal{D}}\xspace}
\newcommand{\calQ}{\ensuremath{\mathcal{Q}}\xspace}
\newcommand\Tbound{\ensuremath{T_{\rm bound}}\xspace}
\newcommand\Sstart{\ensuremath{s_{\textrm{start}}}\xspace}
\algrenewcommand\algorithmicindent{0.5em}

\newtheorem{lemma}{Lemma}
\newtheorem{definition}{Definition}

\title{Alternative Paths Planner (APP) for Provably Fixed-time \\ Manipulation Planning in Semi-structured Environments}

\author{ Fahad Islam$^{1,2}$, Chris Paxton$^{1}$, Clemens Eppner$^{1}$, Bryan Peele$^{1}$, Maxim Likhachev$^{2}$, and Dieter Fox$^{1,3}$%
\thanks{$^{1}$NVIDIA, $^{2}$Carnegie Mellon University, $^{3}$University of Washington.}%
}
\begin{document}

\maketitle

\begin{abstract}
In many applications, including logistics and manufacturing, robot manipulators operate in \emph{semi-structured} environments alongside humans or other robots. These environments are largely static, but they may contain some movable obstacles that the robot must avoid. 
Manipulation tasks in these applications are often highly repetitive, 
 but require fast and reliable motion planning capabilities, often under strict time constraints.
Existing preprocessing-based approaches are beneficial when the environments are highly-structured, but their performance degrades in the presence of movable obstacles, since these are not modelled a priori. 
We propose a novel preprocessing-based method called Alternative Paths Planner (APP) that provides provably fixed-time planning guarantees in semi-structured environments. 
APP plans a set of alternative paths offline such that, for any configuration of the movable obstacles, at least one of the paths from this set is collision-free. During online execution, a collision-free path can be looked up efficiently within a few microseconds.
We evaluate APP on a 7 DoF robot arm in semi-structured domains of varying complexity and demonstrate that APP is several orders of magnitude faster than state-of-the-art motion planners for each domain. We further validate this approach with real-time experiments on a robotic manipulator.
\end{abstract}

\section{Introduction}
In a wide range of environments, be it warehouses, factories or homes, robots share their workspace with humans or other robots. While the obstacle occupancy in such environments is largely fixed or static, yet small regions within these environments may vary in occupancy during operation. We categorise such environments as \emph{semi-structured}. Additionally in these applications robots are often performing repetitive tasks and require fast and reliable motion planning capabilities.

Consider the scenario in Fig.~\ref{fig:cover}. The robot must grasp the bowl while avoiding the pitchers, table and walls.
Most of the environment (e.g. table, walls, free space above the table etc.) can be assumed to be known, and since the task is repetitive some information from previous planning queries can be re-used to speedup planning.
Commonly used algorithms such as RRTs~\cite{lavalle1998rapidly,kuffner2000rrt} neither exploit the distinction between static and movable obstacles nor take advantage of the repetitive nature of planning problems to speedup planning. 

Roadmap-based methods like Probabilistic Roadmaps (PRMs)~\cite{kavraki1996probabilistic} are either limited to fully static environments or incur considerable overheads in repair operations which they require to handle movable obstacles.
Improvements to PRM have minimized expensive collision checks~\cite{kavraki2000path}, accounted for dynamic environments~\cite{leven2002framework,yang2017hdrm}, and captured repetitive motions~\cite{lehner2018repetition}, but their performance can still vary dramatically depending on a planning query.
Instead of preprocessing collision information offline, another class of planners exploits previous experiences to speed up search~\cite{Phillips-RSS-12,berenson2012robot,coleman2015experience}. 
However, these see their performance degrade when the past information is invalidated by the changes in the environment. All these methods therefore can not provide fixed-time planning guarantees. 

Provably fixed-time planners were recently proposed
in~\cite{islam2019provable,islam2020provably} for applications that involve repetitive tasks. Given a start state, a goal region and known model of the environment, they precompute a compressed set of paths that can be utilized online to plan to any goal within the goal region in constant or bounded time. 
These approaches have been used before for mail-sorting~\cite{islam2019provable} and for grasping objects moving down a conveyor belt~\cite{islam2020provably}
While our approach APP draws inspiration from these works, we focus on enabling online obstacle avoidance by explicitly accounting for movable obstacles during the preprocessing phase by using the alternative paths planning approach.

\begin{figure}[bt]
\centering
\includegraphics[width=\columnwidth]{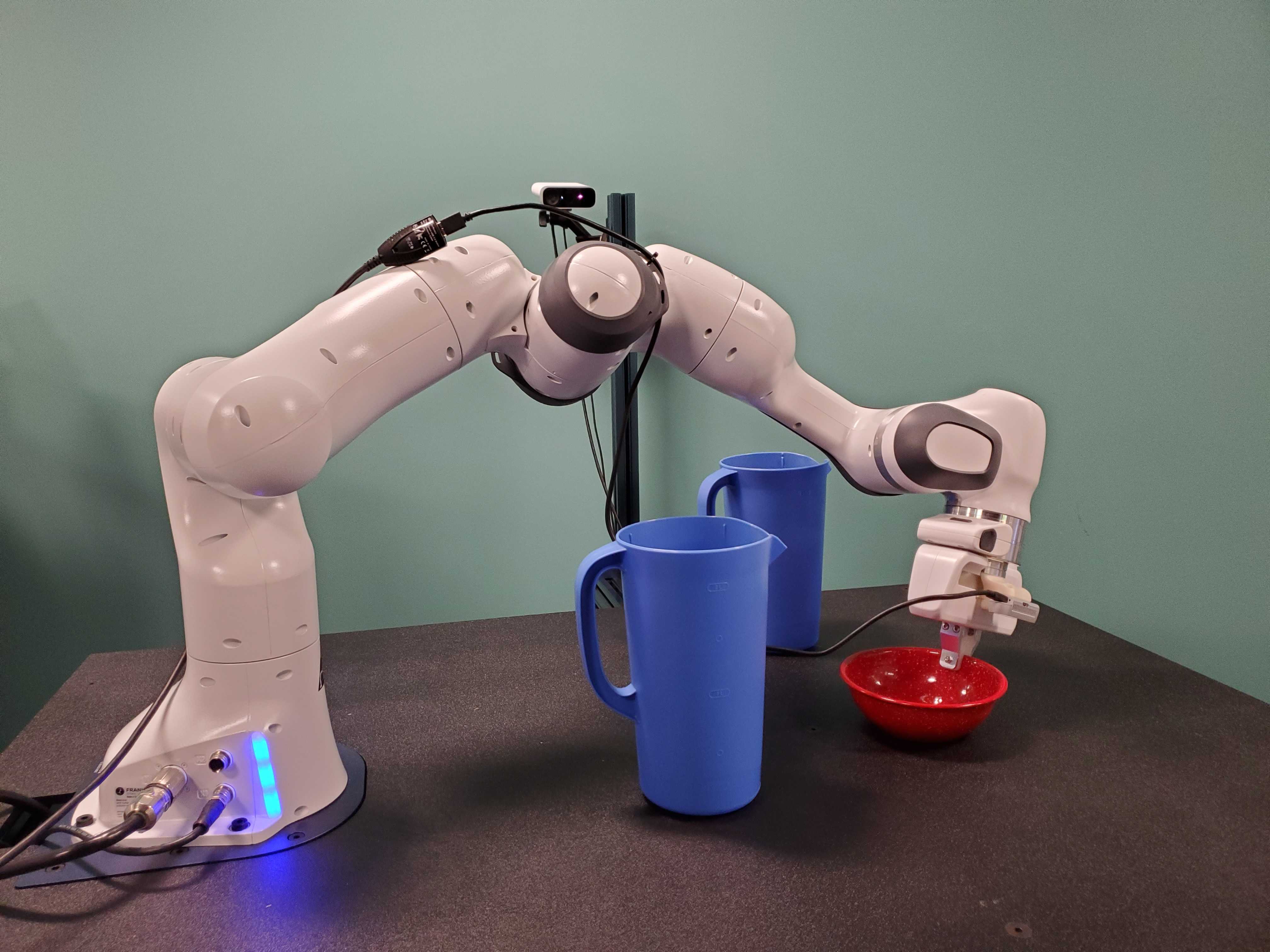}
\caption{Semi-structured tabletop manipulation scenario --- The robot must grasp the bowl while avoiding pitchers and the table itself; the pitchers and bowl may be at any position. The Alternative Paths Planner (APP) guarantees to find a solution for any feasible configuration within few microseconds.}
\label{fig:cover}
\end{figure}

We extend these approaches to handle movable obstacles while still guaranteeing provably fixed-time planning via the \emph{Alternative Paths Planner (APP)}.
The key idea is to generate multiple alternative paths for each planning problem, such that for any configuration of movable obstacles, one path is guaranteed to be traversable.
Our approach to achieve this is to enforce 
that these paths are \emph{disjoint}: they are separated enough that no single obstacle configuration can intersect two paths simultaneously.
Handling $n$ movable obstacles requires generating only at most $n+1$ fully disjoint paths.
However, if the required number of disjoint paths cannot be found for the given problem, we find additional partially overlapping paths (that do not satisfy disjointness criterion) to provide the same guarantee.

APP can be applied to performance-critical real-world use cases such as mail-sorting, conveyor-belt picking, and shelving. 
For scenarios with multiple movable obstacles, we demonstrate microsecond-scale path-lookup times and achieve 100\% success rate. We include experimental results from three domains, compare with state-of-the-art existing approaches and validate our approach with a real-world case study.


\section{Algorithmic Framework}
\label{subsec:setup}
Let \calX be the configuration space of the robot, $\Sstart \in \calX$ be a fixed start configuration of the robot, and \calG be the goal region. 
\calG could be defined in the configuration space \calX (i.e $\calG \subset \calX$) or it could be under-defined, e.g., as the position in $\mathbb{R}^3$ or pose in $SE(3)$ of a target object or the robot end-effector. \calG is discretized to have a finite set of goals $G$. 

Let \calW be a 3D semi-static world in which the robot operates. \calW contains a fixed set of movable obstacles $\calO = \{o_1, o_2, o_3,...,o_n\}$, that occupy $\calW_\calO \subset \calW$, and can be displaced in between different planning queries. A movable obstacle $o_i$ can attain any configuration within a space $\calQ(o_i)$ which is discretized into a set of configurations $Q(o_i)$. Additionally, \calW has a partial occupancy $\calW_S$, which is static i.e $\calW_S = \calW \backslash \calW_\calO$. The motion planning problem is to find a collision-free path for the robot operating in \calW from \Sstart to any goal $g \in G$.


\subsection{Approach}
In this work, we describe an algorithm that provides provable bounds on the planning time for each planning query, where these bounds are small enough to guarantee real-time performance.
We begin by briefly describing a strawman algorithm which solves the aforementioned problem with the provable bounded-time guarantee but is practically prohibitive due to its memory and computational requirements.

\subsubsection{Strawman Algorithm}
Assume that we have access to a motion planner \calP that can be used offline to find feasible paths for the given planning problems. The naive approach
is to precompute and store paths for all possible planning problems using \calP 
offline in a preprocessing stage.
At query time, a lookup table can answer any query in bounded time. Specifically, we precompute the lookup table
$$
\calM : G \times Q(o_1) \times Q(o_2) \times ... \times Q(o_n) \rightarrow \{\pi_1, \pi_2,\pi_3,...\}
$$

\noindent that maps every configuration of goal and the $n$ obstacles to a unique path $\pi_i$ resulting in $|G|.|Q(o_1)|.|Q(o_2)|....|Q(o_n)|$ paths. Assuming that each object $o_i$ can be in the same set of configurations in $\calW$, the space complexity then is $O(|G|.|Q(o_i)|^{n})$ which is exponential in the number of movable objects in \calO.

\subsubsection{Proposed Approach}
While the strawman algorithm provides bounded-time guarantees, its memory and precomputation load is prohibitive for practical purposes. To this end we propose an algorithm that can provide bounded-time guarantees but requires significantly smaller resources in practice.

Our key idea is that, for the given \Sstart and $g$, instead of precomputing path for every possible configuration of obstacles \calO, our algorithm systematically precomputes a small set of paths $\Pi_g$ from \Sstart to $g$ with the guarantee that for any possible configuration of \calO in \calW, at least one path $\pi \in \Pi_g$ will be collision free. At query time, this path can be looked up efficiently in provably bounded time.

\subsection{Algorithm Overview}
Before we describe the approach, we introduce some terminology.

\begin{definition}[Envelope]
\label{def:envp}
For a path $\pi$, an envelope is a set of all obstacle configurations $e_\pi \subset Q(o_1) \cup Q(o_2) \cup...\cup Q(o_n)$, that collide with any robot state $s \in \pi$, except \Sstart, and are not within $\epsilon$ distance of $\textsc{Proj}(g)$
\end{definition}

\noindent where $\textsc{Proj}(g)$ is the projection of $g$ in $\mathbb{R}^3$ and $\epsilon$ is in Euclidean space. Fig.~\ref{fig:envp} illustrates an envelope with a simple example.
The implementation details of envelope construction for the manipulation domain are discussed in Sec.~\ref{sec:impl_details}

\begin{figure}[bt]
\centering
\includegraphics[width=\columnwidth]{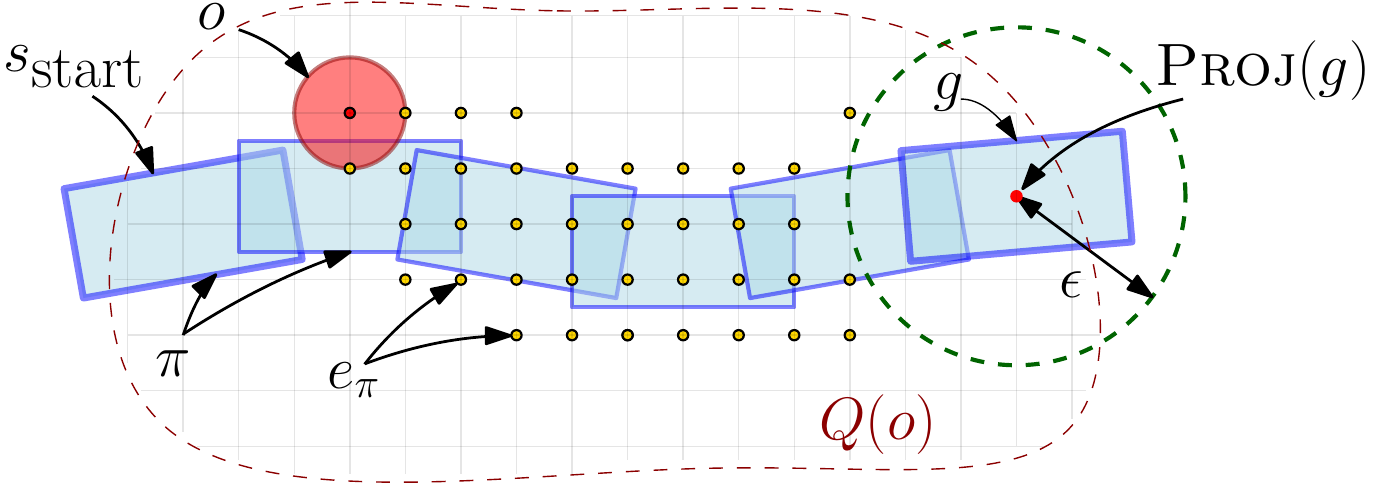}
\caption{Depiction of an envelope for a single movable obstacle $o$: the blue rectangles represent the robot states in $SE(2)$ along the path $\pi$, the red circle is the movable obstacle $o$ that can have any position in $Q(o) \in \mathbb{R}^2$ and the small circles show positions $q(o)$ for which $o$ collides with $\pi$ and thus constitute the envelope $e_\pi$}
\label{fig:envp}
\end{figure}

\begin{definition}[Disjoint paths]
\label{def:disjointness}
The paths $\pi_1,\pi_2,...,\pi_n$ are disjoint if their corresponding envelopes $e_{\pi_1},e_{\pi_2},...e_{\pi_n}$ are disjoint sets.
\end{definition}

Consider finding a collision-free path from \Sstart to a goal $g \in G$ in the presence of a single movable obstacle $o$ and some static occupancy $\calW_S$.
We first find a path $\pi_1$ from \Sstart to $g$ using \calP while avoiding collisions with only $\calW_S$.
Next we construct an envelope $e_{\pi_1}$ around $\pi_1$, where $e_{\pi_1}$ (by Def.~\ref{def:envp}) is the set of all configurations that $o$ can take that invalidate $\pi_1$.
If $e_{\pi_1}$ is non-empty, we attempt to find a second path $\pi_2$ that avoids collisions with $\calW_S$ as well as with the occupancy of $e_{\pi_1}$.
This gives us two disjoint paths $\pi_1$ and $\pi_2$, by Def.~\ref{def:disjointness}.
These two paths constitute the minimal set of paths with the guarantee that for any configuration $q(o) \in Q(o)$, one path will be collision free.
At query time given any configuration $q(o)$, we can check if it lies within $e_{\pi_1}$ or not and use the path $\pi_2$ or $\pi_1$ respectively. By storing envelopes as sets implemented using hash tables, we can check it in constant time~\cite{czech1997perfect}. Note that if $e_{\pi_1}$ is empty, we only require a single path $\pi_1$.

\textbf{Extending to $n$ movable obstacles:} A single movable obstacle, requires at most two disjoint paths. Extending to $n$ objects would require at most $n+1$ disjoint paths. Therefore, for the scenarios in which the required number of disjoint paths exists and \calP can find them within a given allowed planning time, the computational complexity of the preprocessing phase grows linearly with number of movable obstacles i.e $O(|\calO|)$.
If \calP fails to find the required number of disjoint paths, however,
then this method is incomplete as is, and it would require more preprocessing efforts to find more paths to satisfy the criterion that for any possible configuration of the set of obstacles \calO, at least one of the paths from the set is guaranteed to be collision free. 

\subsection{Algorithm Details}
We now describe the two phases of our algorithm: the preprocessing and the query phases.

\subsubsection{Preprocessing}
The preprocessing algorithm is described in Alg.~\ref{alg1} and is illustrated in Fig.~\ref{fig:illustration}. It starts with finding the first collision-free path $\pi_1$ from \Sstart to the given goal $g \in G$ in $\calW_S$. For each computed path $\pi_j$, the algorithm maintains the set of envelopes $\calE_{\pi_j}$ that were avoided while computing $\pi_j$. For the first path $\pi_1$, the set $\calE_{\pi_1}$ is empty since \calP only considers $\calW_S$ for this step. The algorithm then runs a loop for $n=|\calO|$ iterations, attempting to find $n$ more mutually disjoint paths (loop at line~\ref{alg1:l1}). Another loop at line~\ref{alg1:l2} is needed for the case when $n+1$ disjoint paths do not exist and therefore, more than one paths are computed in a single iteration of loop at line~\ref{alg1:l1}.
In every iteration of the loop at line~\ref{alg1:l2}, the algorithm first constructs the envelope $e_{\pi_j}$ around the path $\pi_j$ found in the previous iteration of loop at line~\ref{alg1:l1}. Next, $\pi_j'$ is computed while avoiding collisions with the occupancy of the set of envelopes $\calE_{\pi_j}'$, which is the union of $e_{\pi_j}$ and the envelopes in $\calE_{\pi_j}$. $\calE_{\pi_j}'$ thus constitutes the set of avoided envelopes for $\pi_j'.$

For the case when \calP fails to find the disjoint path at any iteration of loop at line~\ref{alg1:l1}, the algorithm bisects one of the envelopes (line~\ref{alg1:bisect} and Alg.~\ref{alg2}) that \calP attempted to avoid, resulting in two sets of envelopes (each one containing a bisected envelope along with the remaining envelopes). \calP then tries to find paths around the occupancy of envelopes in each of the two sets independently. Note that this process (Alg.~\ref{alg2}) repeats recursively until either \calP successfully finds paths around all the newly created sets of envelopes or until further recursion is not possible. The latter occurs in the worst case, when the algorithm recurses down to the deepest level, where each envelope contains individual obstacle configurations. Owing to the structure of these envelopes, we implement them as binary-trees. Because of this envelope bisection, the algorithm may compute more than one path in a single iteration of loop at line~\ref{alg1:l1}. It therefore maintains a set of paths $\Pi_i$ for each iteration and attempts to compute disjoint paths for all paths $\pi_j \in \Pi_i$ in the following iteration.

The preprocessing algorithm (Alg.~\ref{alg1}) returns a database of paths $\Pi_g = \Pi_1 \cup \Pi_2 \cup ... \cup \Pi_{n+1}$ from \Sstart to a goal $g$ with the guarantee that for any possible configuration of obstacles \calO, one of the paths $\pi \in \Pi_g$ is collision free. To cover the full goal region $G$, Alg~\ref{alg1} is called for each $g \in G$ in the preprocessing phase.

\begin{figure*}[t]
    \centering
    \begin{subfigure}{0.24\textwidth}
        \includegraphics[width=\textwidth]{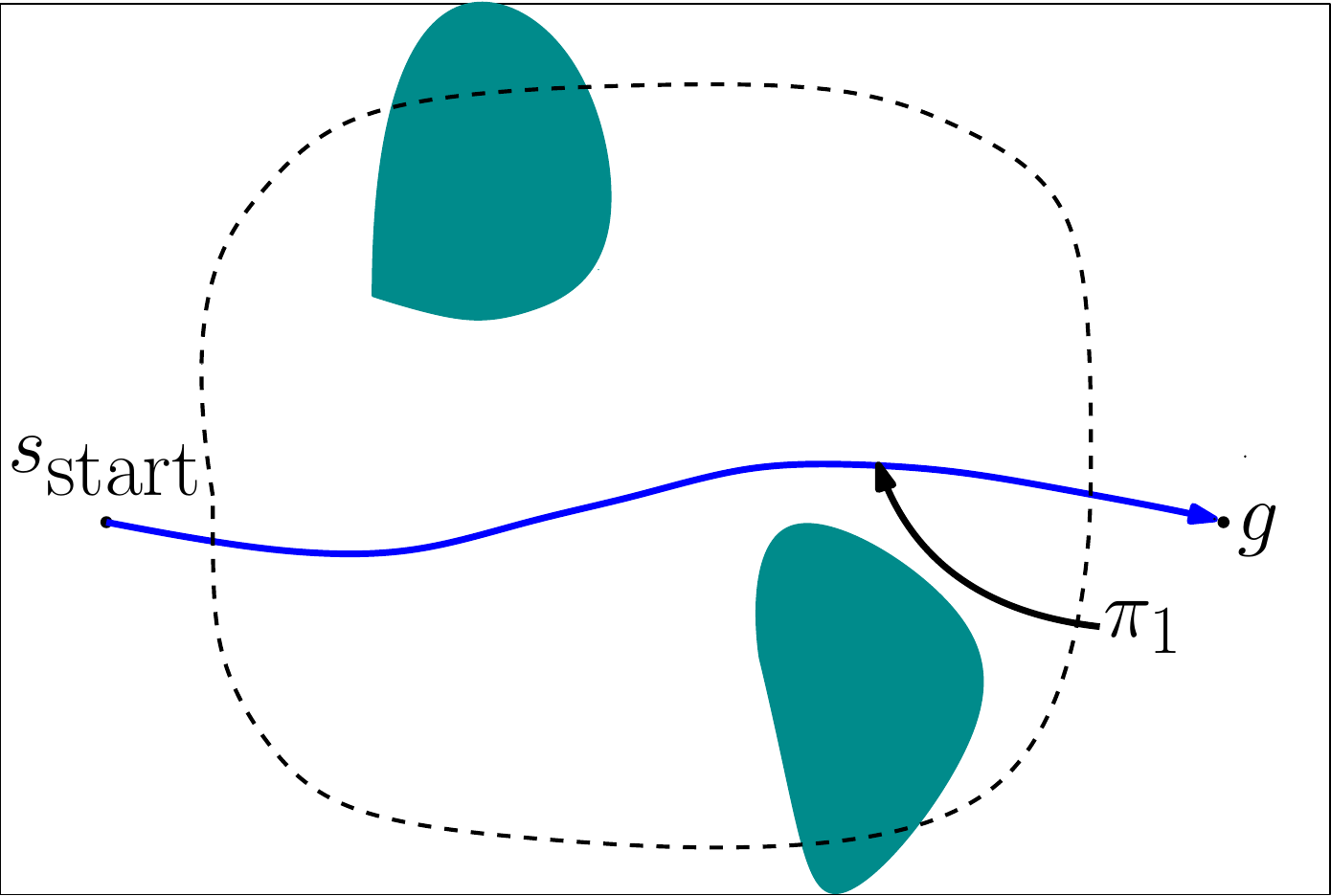}
        \caption{}
        \label{fig:p1}
    \end{subfigure} 
    \begin{subfigure}{0.24\textwidth}
        \includegraphics[width=\textwidth]{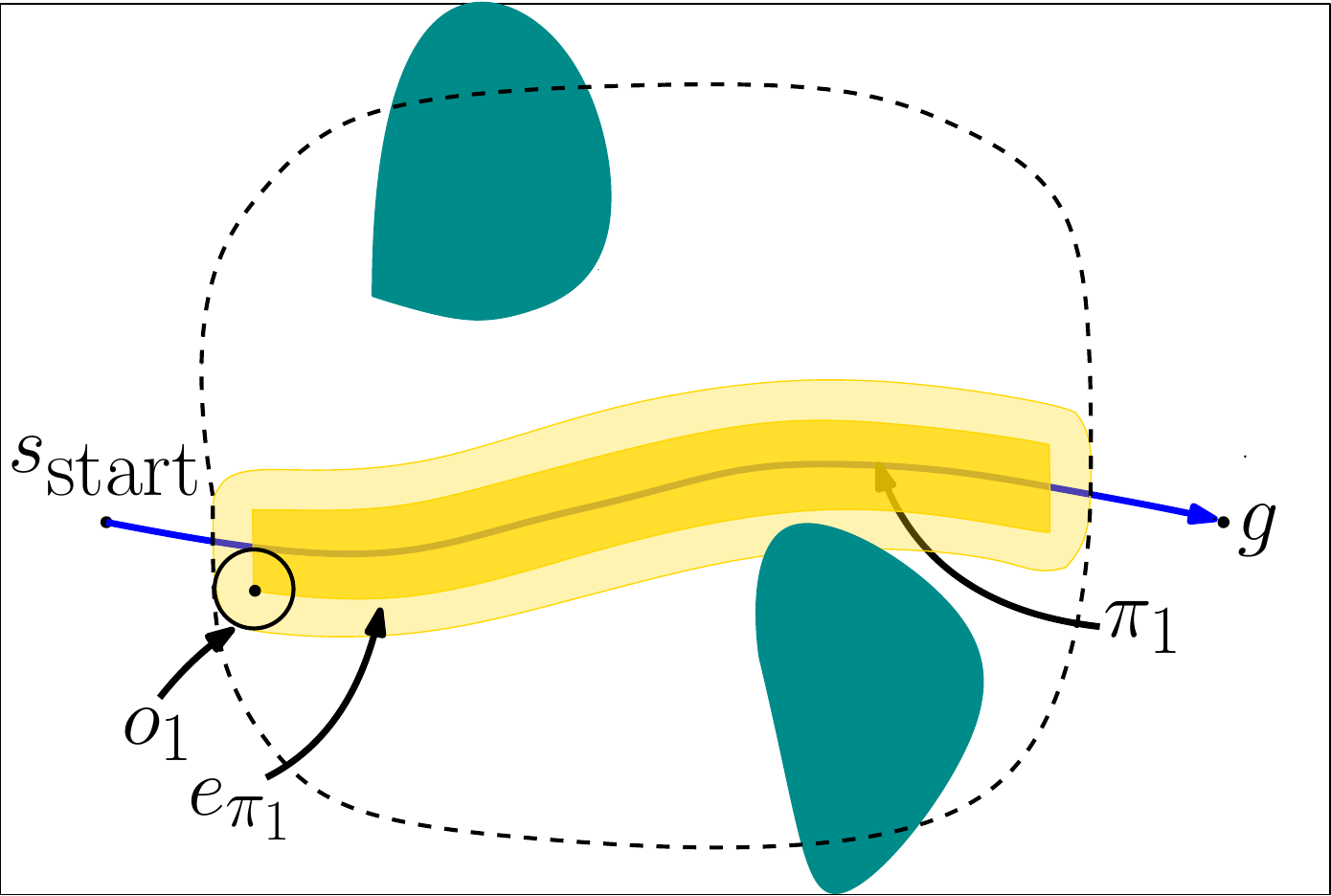}
        \caption{}
        \label{fig:p2}
    \end{subfigure}
    \begin{subfigure}{.24\textwidth}
        \includegraphics[width=\textwidth]{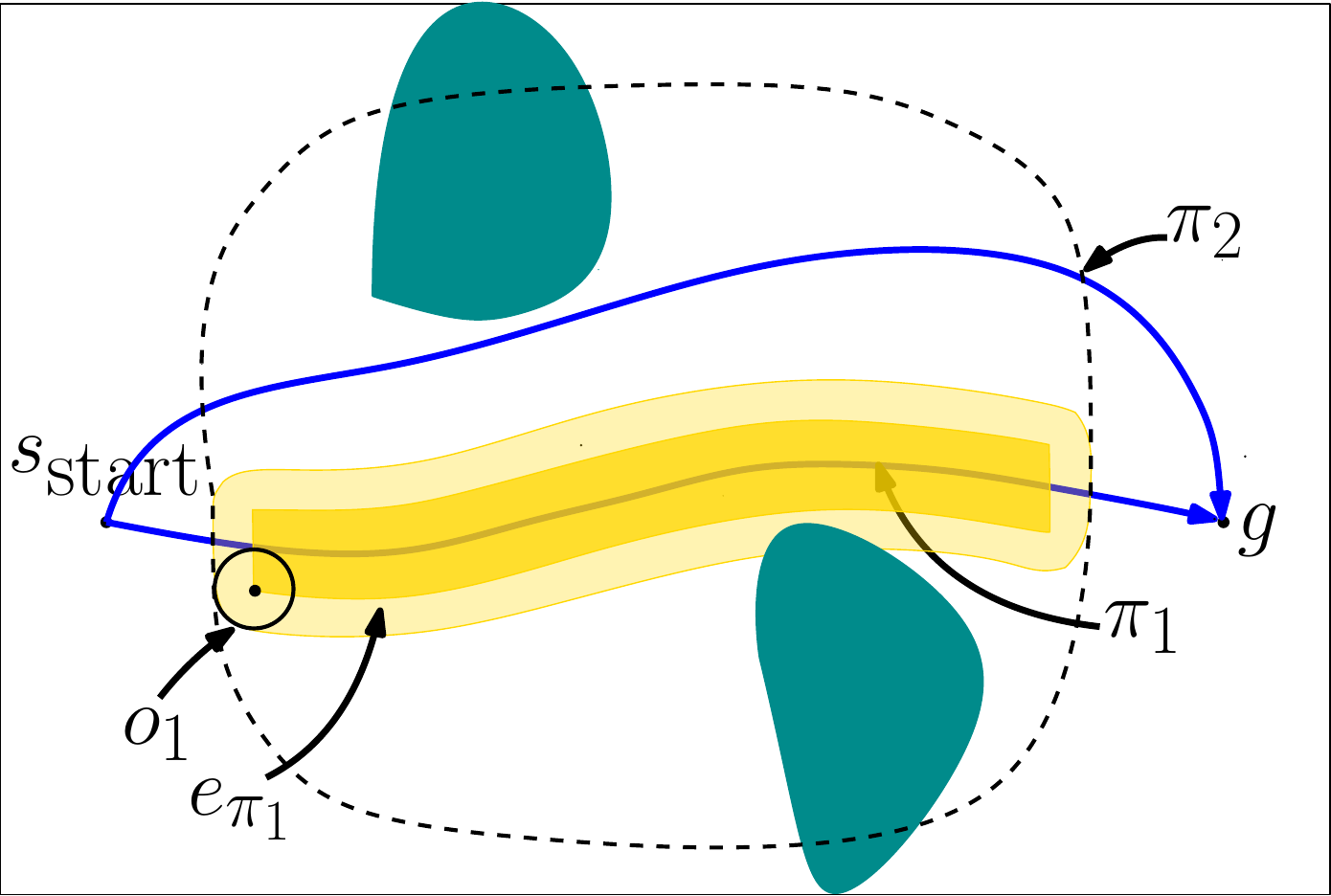}
        \caption{}
        \label{fig:p3}
    \end{subfigure}
    \begin{subfigure}{.24\textwidth}
        \includegraphics[width=\textwidth]{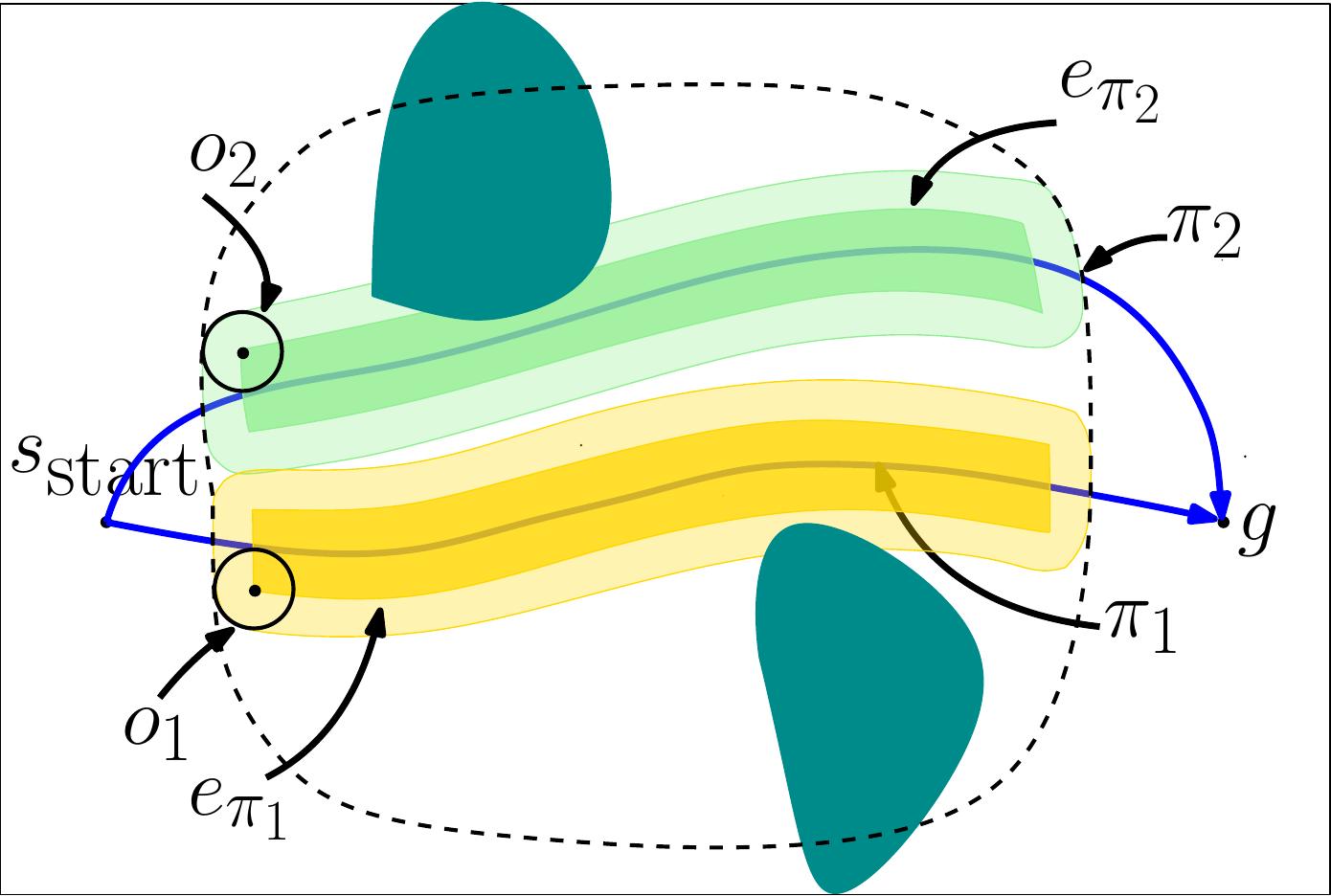}
        \caption{}
        \label{fig:p4}
    \end{subfigure}
        \begin{subfigure}{.24\textwidth}
        \includegraphics[width=\textwidth]{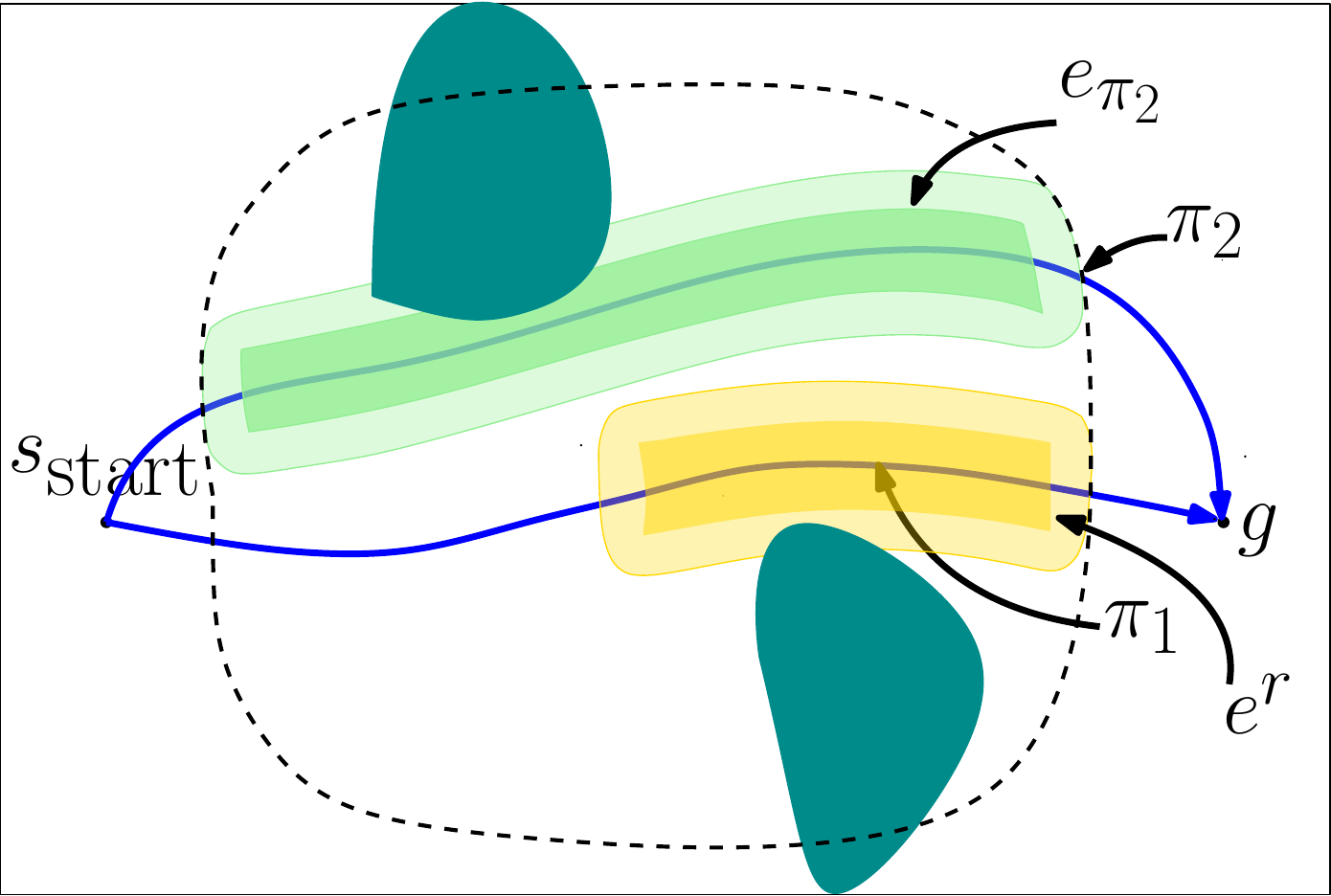}
        \caption{}
        \label{fig:p5}
    \end{subfigure}
        \begin{subfigure}{.24\textwidth}
        \includegraphics[width=\textwidth]{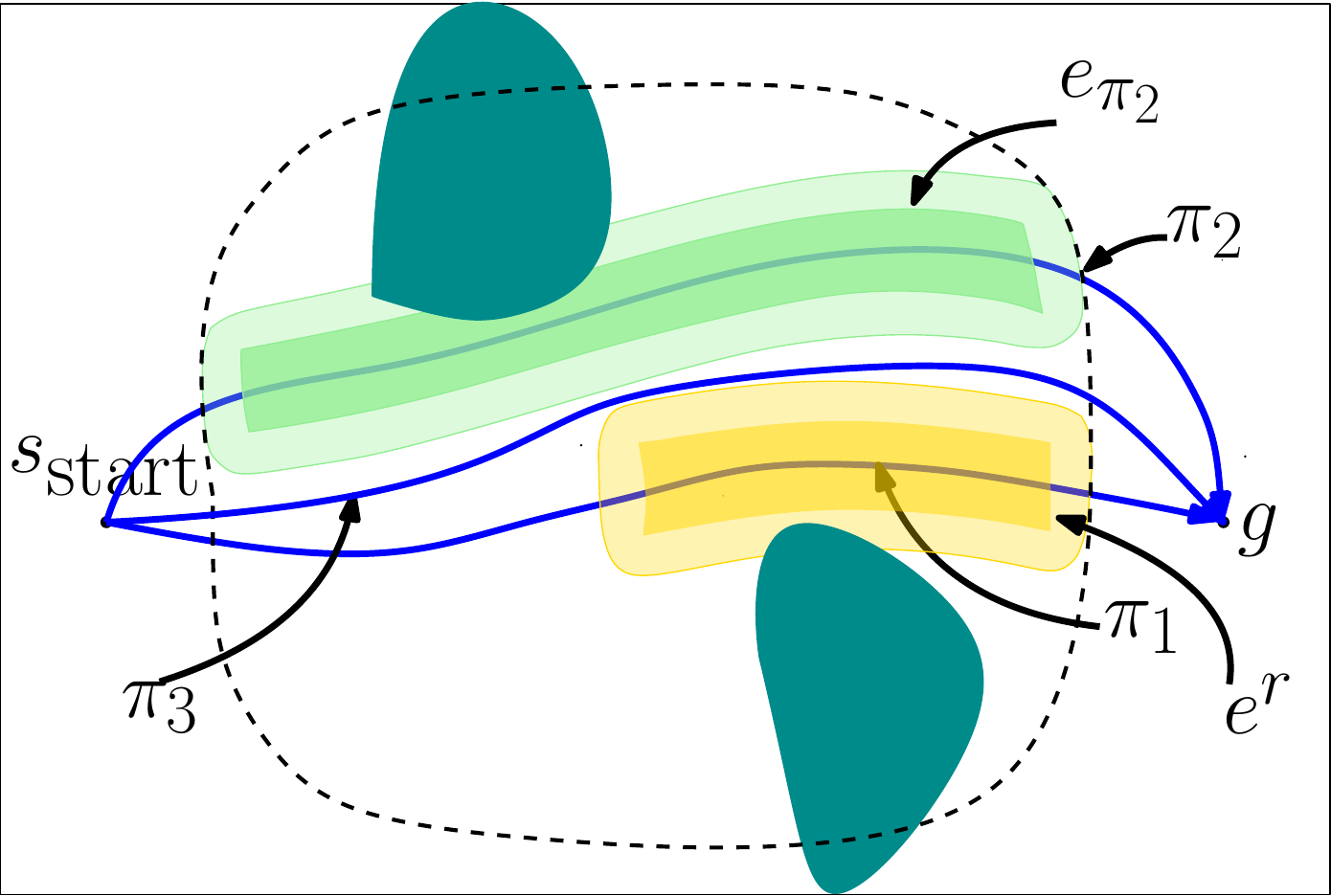}
        \caption{}
        \label{fig:p6}
    \end{subfigure}
        \begin{subfigure}{.24\textwidth}
        \includegraphics[width=\textwidth]{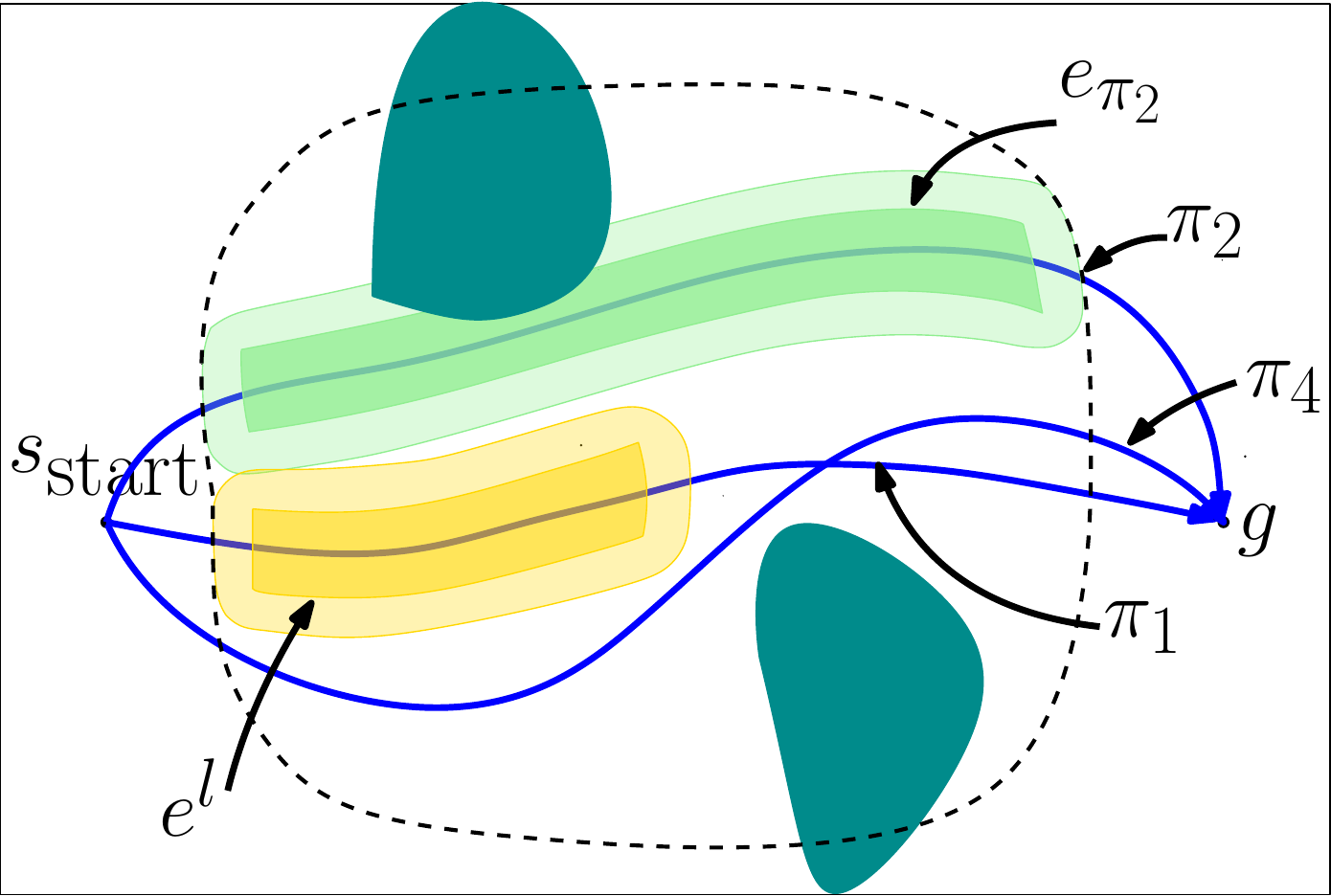}
        \caption{}
        \label{fig:p7}
    \end{subfigure}
        \begin{subfigure}{.24\textwidth}
        \includegraphics[width=\textwidth]{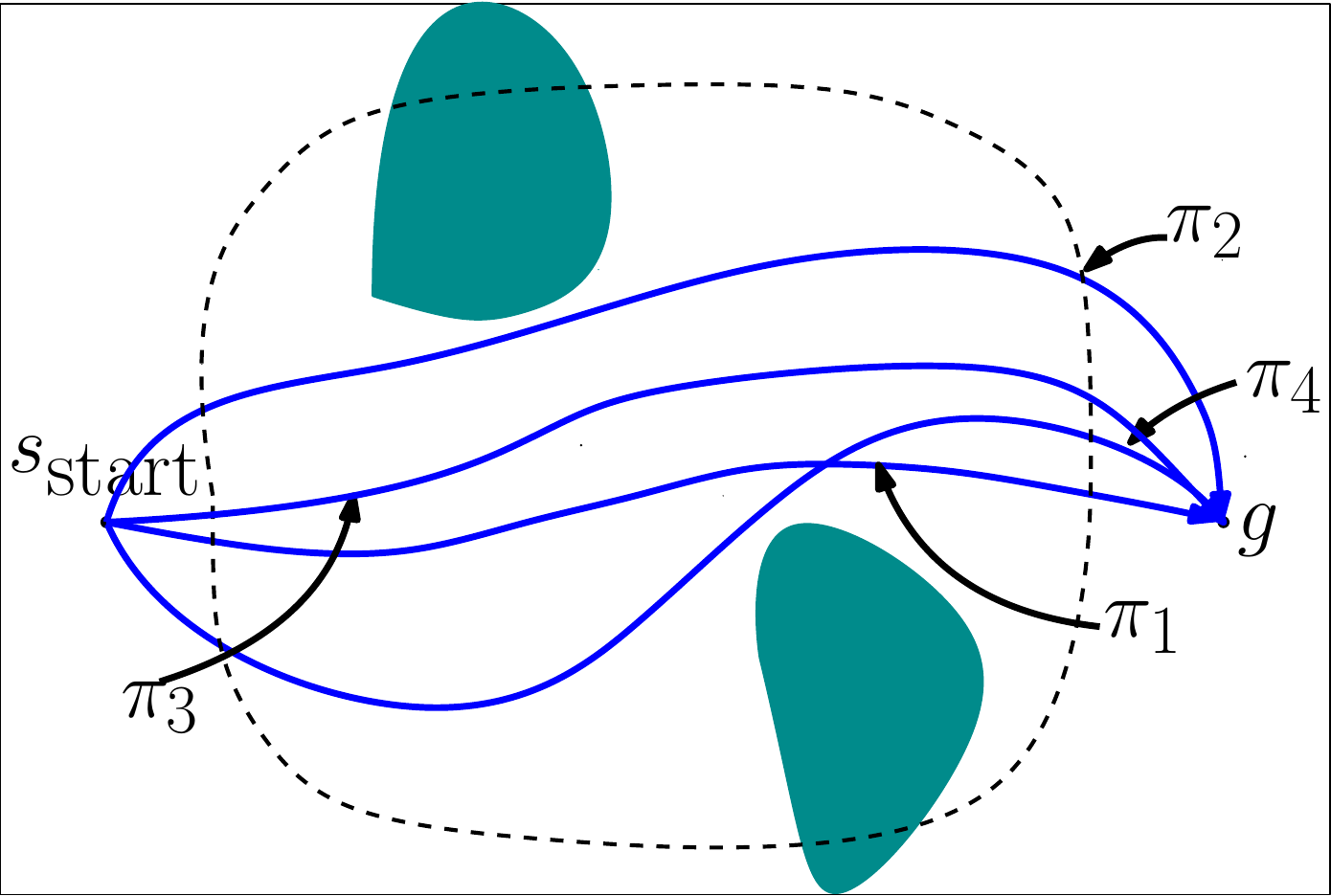}
        \caption{}
        \label{fig:p8}
    \end{subfigure}
    \caption{
    Illustration of preprocessing steps in a 2D environment with two movable obstacles $\{o_1,o_2\} \in \calO$ for a point robot. APP plans a set of paths from \Sstart to $g$ such that for any configuration of $o_1,o_2$ at least one of the paths from the set is collision free. Static obstacles $\calW_S$ are shown in green and $o_1,o_2$ are depicted as circles. The dotted region depicts $\calQ(o_i)$ (identical for $o_1,o_2$). Namely, $o_1,o_2$ are confined within the dotted region.
     (\subref{fig:p1})~APP finds the first path $\pi_1$ avoiding $\calW_S$ only
    (\subref{fig:p2})~Constructs envelope $e_{\pi_1}$ around $\pi_1$. The dark shaded region shows $e_{\pi_1}$ and the light region shows its occupancy.
    (\subref{fig:p3})~Finds second path $\pi_2$ avoiding $e_{\pi_2}$. Since APP successfully finds two disjoint paths at this step, these paths are sufficient to avoid any configuration of $o_1$.
    (\subref{fig:p4})~Constructs envelope $e_{\pi_2}$ around $\pi_2$.
    (\subref{fig:p5})~Attempts to find the third disjoint path but fails and bisects $e_{\pi_1}$ into $e^l$ and $e^r$.
    (\subref{fig:p6})~Finds path $\pi_3$ around $e^l$ and $e_{\pi_2}$
    (\subref{fig:p7})~Finds path $\pi_4$ around $e^r$ and $e_{\pi_2}$ 
    (\subref{fig:p8})~Set of paths needed to avoid any configuration of $o_1,o_2$.
    }

    \label{fig:illustration}
\end{figure*}

\newlength{\textfloatsepsave}
\setlength{\textfloatsepsave}{\textfloatsep}
\setlength{\textfloatsep}{0pt}
\begin{algorithm}[t]
\caption{\textsc{Preprocess($g$)}} \label{alg1}

\hspace*{\algorithmicindent} \textbf{Inputs:} $\Sstart, \calW_S, \calO$

\hspace*{\algorithmicindent} \textbf{Output:} $\Pi_g$
\begin{algorithmic}[1]
\State $\pi_1 \leftarrow$ \textsc{FindPath}($\Sstart,g,\calW_S$)
\State $\calE_{\pi_1} \leftarrow \emptyset$
\State \textsc{SetAvoidedEnvelopesOfPath($\pi_1,\calE_{\pi_1}$)}
\State $\Pi_1 \leftarrow \{\pi_1\}$
\State $\Pi_g \leftarrow \Pi_1$
\For {$i = 1$ to $n$}  \Comment{$n = |\calO|$}  \label{alg1:l1}
    \State $\Pi_{i+1} \leftarrow \emptyset$
    \For{each $\pi_j \in \Pi_i$} \label{alg1:l2}
    \State $e_{\pi_j} \leftarrow$ \textsc{ConstructEnvelope($\pi_j$)}
    \State \textbf{If} $e_{\pi_j} = \emptyset$ \textbf{skip to next iteration}
        \State $\calE_{\pi_j} \leftarrow$ \textsc{GetAvoidedEnvelopesOfPath}($\pi_j$)
        \State $\calE_{\pi_j'} \leftarrow \calE_{\pi_j} \cup \{e_{\pi_j}\}$
        \State $\pi_j' \leftarrow$ \textsc{FindPath}($\Sstart,g,\calW_S,\calE_{\pi_j'}$) \Comment{Avoiding $\calE_{\pi_j'}$}
        \If{success}
            \State \textsc{SetAvoidedEnvelopesOfPath($\pi_j', \calE_{\pi_j'}$)}
            \State $\Pi_{i+1} \leftarrow \Pi_{i+1} \cup \{\pi_j'\}$
        \Else
            \State $\Pi_{i+1} \leftarrow \Pi_{i+1} \cup \textsc{BisectAndFindMorePaths}(\calE_{\pi_j'})$ \label{alg1:bisect}
    \EndIf
    \EndFor
    \State $\Pi_g \leftarrow \Pi_g \cup \Pi_{i+1}$
\EndFor
\end{algorithmic}
\end{algorithm}
\setlength{\textfloatsep}{\textfloatsepsave}

\setlength{\textfloatsepsave}{\textfloatsep}
\setlength{\textfloatsep}{0pt}
\begin{algorithm}[t]
\caption{\textsc{BisectAndFindMorePaths($\calE$)}} \label{alg2}
\begin{algorithmic}[1]
    \State $\Pi \leftarrow \emptyset$
    \State $e \leftarrow \calE$.pop()     \Comment{pop an envelope to bisect}
    \label{alg2:pop}
    \If {\textsc{CheckSingleton}($e$)}    \Comment{contains one position}
        \State \textbf{return} $\emptyset$  \Comment{no further bisection possible}
    \EndIf
    \State $e^{l},e^{r} \leftarrow$ \textsc{BisectEnvelope($e$)} \label{alg:2:bisect}
    \State $\calE_{\pi_l} \leftarrow \calE \cup \{e^l\}$
    \State $\calE_{\pi_r} \leftarrow \calE \cup \{e^r\}$
    \For {each $\calE_{\pi_i} \in \{\calE_{\pi_l},\calE_{\pi_r}\}$}
    
        \State $\pi_i \leftarrow$ \textsc{FindPath}($\Sstart,g,\calW_S,\calE_{\pi_i'}$) \Comment{Avoiding $\calE_{\pi_i'}$}
        \If{success}
            \State \textsc{SetAvoidedEnvelopesOfPath($\pi_i, \calE_{\pi_i}$)}
            \State $\Pi \leftarrow \Pi \cup \{\pi_i\}$
        \Else
            \State $\Pi \leftarrow \Pi \cup \textsc{BisectAndFindMorePaths}(\calE_{\pi_i})$
        \EndIf
    \EndFor
    \State \textbf{return} $\Pi$
    
\end{algorithmic}
\end{algorithm}
\setlength{\textfloatsep}{5pt}

\subsubsection{Query}
A query is comprised of a goal $g \in G$ (\Sstart is fixed) and obstacle configurations $q(o_1) \in Q(o_1), q(o_2) \in Q(o_2),..., q(o_n) \in Q(o_n)$. The query phase first looks up the datastructures stored for the queried $g$. It then loops over the set of paths $\Pi_g$ until it finds a collision free path. It is important to note that while doing so, it does not require any collision checking, which typically is the most computationally expensive operation in motion planning. Instead, it uses the envelope $e_{\pi_j}$ of each path $\pi_j$ to find the collision free path. The following data structures are stored for the query phase to efficiently return collision free paths.
\begin{center}
\begin{tabular}{l}
    $\calM : \; [G \rightarrow \{\calD_{g_1}, \calD_{g_2},...\}]$ \\
    $\calD_{g_i} : \; <\{e_{\pi_1},e_{\pi_2},...\} : \{\pi_1,\pi_2,...\}>$ \Comment{for each $g_i \in G$}\\
\end{tabular}{}
\end{center}

Namely, $\calM$ is a lookup table that maps a goal~$g_i \in G$ to a dictionary~$\calD_{g_i}$ where ~$\calD_{g_i}$ contains $<e_{\pi_j},\pi_j>$ pairs where $\pi_j \in \Pi_{g_i}$. The query phase is a straight-forward two step process; in the first step, the dictionary~$\calD_{g_i}$ is fetched for the queried $g_i$ from $\calM$. In the second step, the algorithm loops through all the pairs in $\calD_{g_i}$ to search for an envelope $e_{\pi_j}$ that \emph{does not} contain any of the obstacle configurations i.e. $e_{\pi_j} | q(o_1),q(o_2),...,q(o_n) \notin e_{\pi_j}$ in the given planning query. By construction of the preprocessing algorithm, atleast one of the envelopes in $\calD_{g_i}$ is guaranteed to satisfy this condition.  To efficiently check whether a configuration is contained in an envelope or not, we store envelopes as hash tables, which allows the algorithm to perform this check in constant time.

\subsection{Theoretical Guarantees}
\begin{lemma}[Fixed-time query]
The query time of APP is upper-bounded by a (small) time limit $t_{\textrm {query}}$.
\label{lemma1}
\end{lemma}
\begin{proof}[Sketch of Proof]
The query phase includes a lookup operation on \calM, running a loop over the set of alternative paths $\Pi_g$ to the queried $g$ and a nested loop through \calO which also only involves lookup operations since the envelopes are stored as sets. Since with perfect hashing (~\cite{czech1997perfect}), the lookup operations can be performed in constant time. The overall complexity thus is bounded and is given by $O(|\Pi_g|\cdot |\calO|)$ which bounds $t_{\textrm {query}}$.
\end{proof}

\textbf{Special case:} When all paths in $\Pi_g$ are disjoint, then we have that $|\Pi_g| \leq n+1$ and so the complexity of the query phase becomes $O(|\calO|^2)$

In practice we have that $|\Pi_g| \ll Q(o_1) \times Q(o_2) \times ... \times Q(o_n)$ (number of paths required for the naive approach) and therefore, $t_{\textrm {query}}$ is very small.

\begin{lemma}[Completeness]
If the offline motion planner \calP can find a collision-free path from \Sstart to a $g \in G$ in a (large) timeout $t_{\textrm preprocess}$, then APP is guaranteed to generate a path for the same planning problem in time $t_{\textrm {query}}$.
\end{lemma}
\begin{proof}[Sketch of Proof]
For the special case when all paths in $\Pi_g$ are disjoint, the proof follows from Def.~\ref{def:disjointness} because, since the envelopes of paths in $\Pi_g$ are disjoint sets, no obstacle configuration $q(o_i)$ can intersect with more than one path in $\Pi_g$. Since the algorithm computes $n+1$ paths if required for $n$ obstacles, at least one path is collision free.

For the general case where not all paths in $\Pi_g$ are necessarily disjoint, the lemma can be proven by assessing the recursive bisection procedure in APP. In the Alg.~\ref{alg1}, \calP attempts to find the ($n+1$)th path which avoids $n$ envelopes and bisects one of the $n$ envelopes if \calP fails within $t_{\textrm {preprocess}}$. The bisection occurs recursively until either the path is computed for every leaf envelope or further recursion is not possible, thus covering all possible configurations of obstacles in \calO. Since APP uses \calP with the timeout $t_{\textrm {preprocess}}$ to solve each problem, it guarantees to have precomputed a valid path for any problem that \calP can solve in $t_{\textrm {preprocess}}$.
Furthermore from Lemma~\ref{lemma1}, we have that the query time is guaranteed to be within $t_{\textrm {query}}$.
\end{proof}

\section{Implementation Details} 
In this section we cover three main components of our implementation for the motion planning problem of a high-DoF robot arm, (1) envelope construction, (2) envelope bisection and (3) computing envelope occupancy.
Our implementation approximates the geometry of the movable obstacles with spheres. This approximation restricts the dimensionality of $Q(o_i)$ to $\mathbb{R}^3$.

\subsection{Envelope Computation using Distance Field}
\label{sec:impl_details}
The naive approach of computing an envelope $e_\pi$ around a path $\pi$ (recall Def.~\ref{def:envp}) requires collision checking all obstacle configurations with each robot state $s \in \pi$. This can be very expensive for a large number of configurations that the obstacles can have.
%
%
%
%
%
%
Instead, the envelope is constructed in a simple two step process.
First, the path $\pi$ is voxelized (using the collision model of the robot) and added to an occupancy grid.
Second, 
to compute the positions of each $o_i \in \calO$ that collide with $\pi$, the occupancy grid is inflated by the radius~$r_i$ of $o_i$, (approximating an obstacle $o_i$ with a sphere of radius $r_i$) which is efficiently done by using a signed distance field.
%
The intersection of the set of discrete positions of the occupied cells with $Q(o_i)$, constitutes $e_{\pi}$.
%
To allow computing disjoint paths, following Def.~\ref{def:envp}, the obstacle positions that collide with \Sstart or are within a small distance $\epsilon$ from the goal position are excluded from $e_\pi$. This enforces an assumption that an obstacle is not placed within $\epsilon$ distance from the goal position.

\subsection{Envelope Bisection}
The envelope bisection refers to splitting an envelope~$e$ into two sub-envelopes~$e_l$ and $e_r$ (see Alg.~\ref{alg2} line~\ref{alg:2:bisect}). Several schemes could be used to make this split. In our implementation, we split the envelope along either of the three planes $x=x_c, y=y_c$ or $z=z_c$. where the $x_c,y_c$ and $z_c$ are the means of the $x,y$ and $z$ components respectively of all positions in the envelope $e$ to be bisected. Among the three axes, we pick the axis that has the largest span of positions. As described earlier, each envelope is implemented as a binary tree and this bisection results in creation of two children~$e_l$ and~$e_r$ for the parent~$e$ in this binary tree. Note that we remove subscripts for paths once the envelope is bisected since the bisected envelopes no more follow Def.~\ref{def:envp}.
Alg.~\ref{alg2} also makes a choice at line~\ref{alg2:pop} for the envelope to be bisected from the input set of envelopes \calE. While this decision does not affect the properties of the algorithm, different heuristics can be used for this choice as well. In our implementation, we simply select the largest envelope for bisection.

\subsection{Computing Envelope Occupancy}
This step corresponds to an implementation detail within procedure \textsc{FindPath} in Alg.~\ref{alg1}. In order to find a path around the set of envelopes \calE, the joint occupancy of all envelopes in \calE is computed. To do so, for each obstacle $o_i$, all of its discrete positions in \calE are added to an empty occupancy grid, followed by the inflation of the occupancy grid (as described earlier) by the radius $r_i$ of the corresponding obstacle $o_i$. The joint occupancy of all these grids constitutes the occupancy of \calE.
\footnote{In our experiments, we consider all obstacles of the same sizes. In that case, a single inflation operation is needed, each for the envelope construction and computing envelope occupancy steps.}

\begin{figure*}[t]
    \centering
    \begin{subfigure}{0.24\textwidth}
        \includegraphics[trim={8cm 0 10cm 1cm},clip, width=\textwidth]{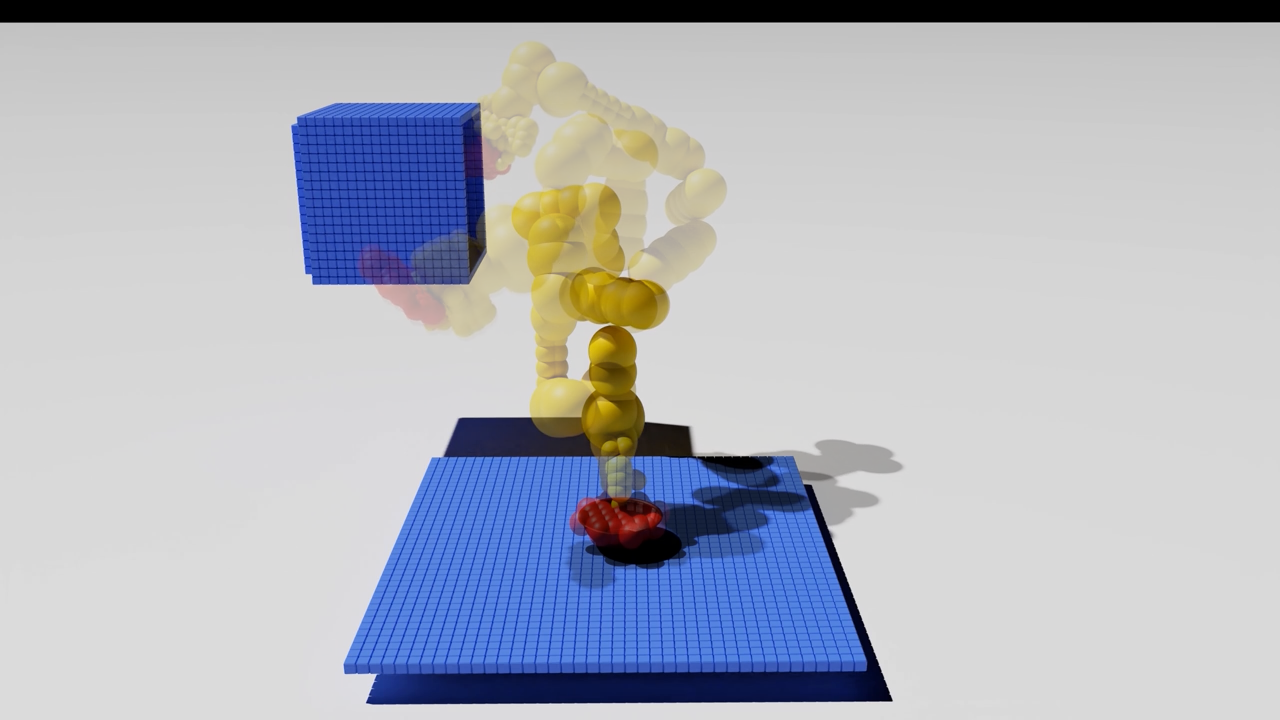}
        \caption{}
        \label{fig:pr1}
    \end{subfigure} 
    \begin{subfigure}{0.24\textwidth}
        \includegraphics[trim={8cm 0 10cm 1cm},clip, width=\textwidth]{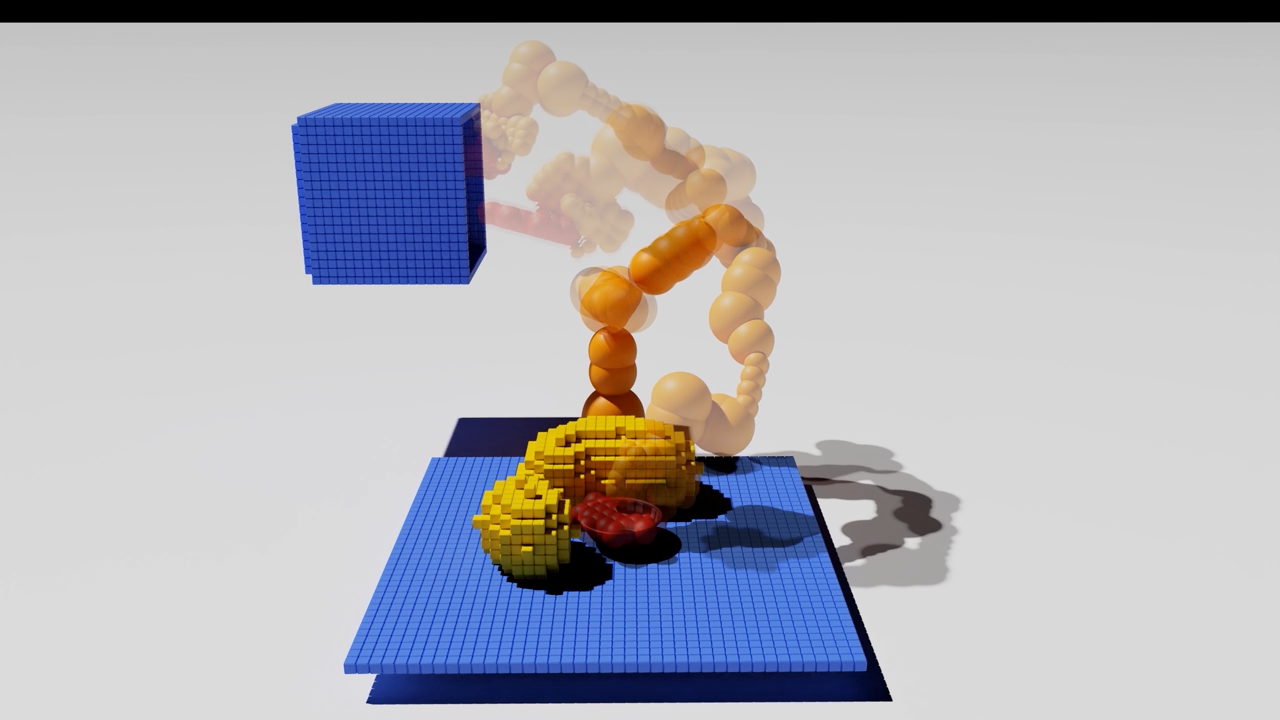}
        \caption{}
        \label{fig:pr2}
    \end{subfigure}
    \begin{subfigure}{.24\textwidth}
        \includegraphics[trim={8cm 0 10cm 1cm},clip, width=\textwidth]{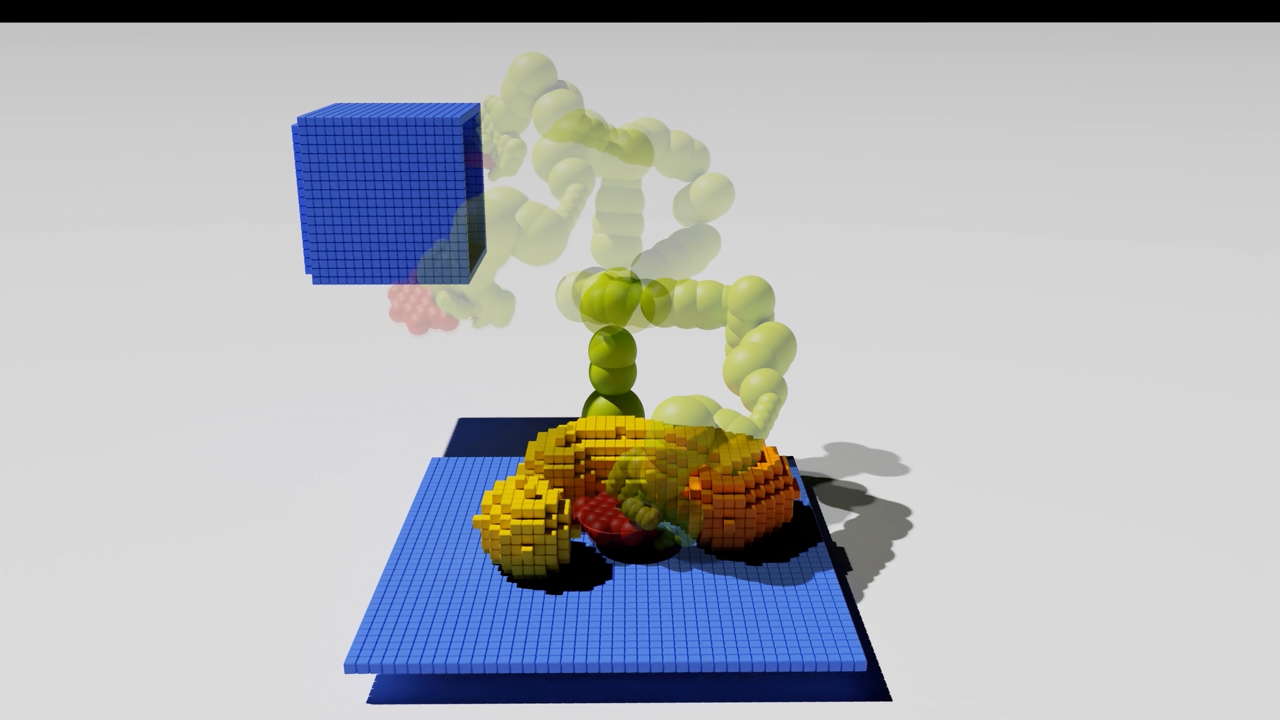}
        \caption{}
        \label{fig:pr3}
    \end{subfigure}
    \begin{subfigure}{.24\textwidth}
        \includegraphics[trim={8cm 0 10cm 1cm},clip, width=\textwidth]{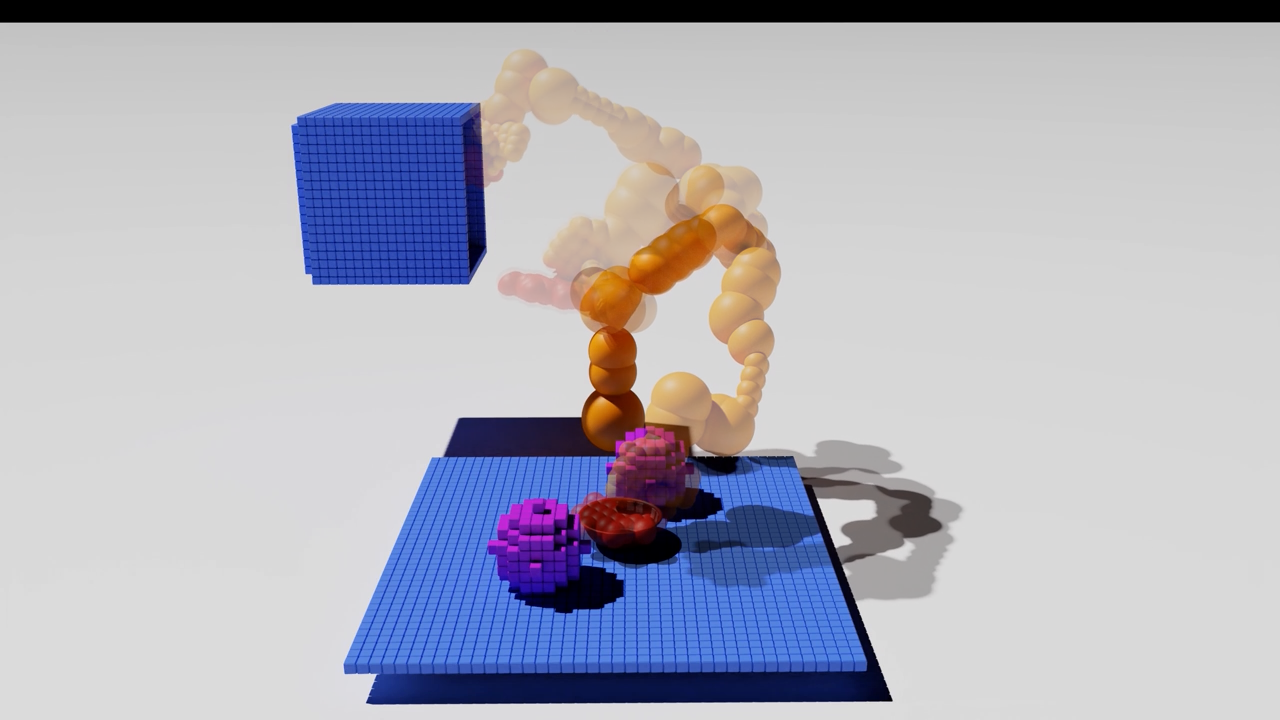}
        \caption{}
        \label{fig:prq}
    \end{subfigure}
    \caption{
    Illustration of the preprocessing steps for the mail-sorting environment with two movable obstacles on a tabletop for a single start and goal pair.
     (\subref{fig:pr1})~APP finds the first path avoiding $\calW_S$ only.
    (\subref{fig:pr2})~Constructs envelope around first path and finds the second path avoiding the first envelope.
    (\subref{fig:pr3})~Constructs a second envelope around the second path and finds the third path avoiding the two envelopes.
    (\subref{fig:prq})~At query time, given a random configuration of movable obstacles, APP looks up a valid path~(second path, Fig.~\subref{fig:pr2}) among the set of precomputed paths.
    }
    \label{fig:pr}
\end{figure*}

\section{Experiments}
\label{sec:experiments}
\begin{table*}[tbh!]
\footnotesize
\begin{center}
\begin{tabular}{|r|r||c|c|c|c|c|c|c|c|c|}
\hline
\multicolumn{2}{|l|}{} & \multicolumn{3}{c|}{Shelving} & \multicolumn{3}{c|}{Sorting} & \multicolumn{3}{c|}{Sorting (time constrained)} \\
\multicolumn{2}{|l|}{} & \multicolumn{3}{c|}{\scriptsize{[248 goals]}} & \multicolumn{3}{c|}{\scriptsize{[680 goals]}} & \multicolumn{3}{c|}{\scriptsize{[680 goals]}} \\
\multicolumn{2}{|r|}{No. of Obstacles} & 1 & 2 & 3 & 1 & 2 & 3 & 1 & 2 & 3 \\ \hline\hline
\multirow{7}{*}{\rotatebox[origin=c]{90}{Success Rate [\%]}}
& \textbf{APP}  & \textbf{100}  & \textbf{100}  & \textbf{100}  & \textbf{100}  & \textbf{100}  & \textbf{100}  & \textbf{100}  & \textbf{100}  & \textbf{100} \\ \cline{2-11}
& Lightning~\cite{berenson2012robot}  & 94  & 95  & 97  & \textbf{100}  & \textbf{100}  & \textbf{100}  & 97  & 96  & 75 \\ \cline{2-11}
& E-Graphs~\cite{Phillips-RSS-12}  & 20  & 15  & 14  & 89  & 89  & 81  & 50  & 45  & 41 \\ \cline{2-11}
& RRTConnect~\cite{kuffner2000rrt}  & 96  & 95  & 97  & \textbf{100}  & \textbf{100}  & \textbf{100}  & 19  & 14  & 11 \\ \cline{2-11}
& LazyPRM~\cite{kavraki2000path}  & 46  & 34  & 23  & 98  & 99  & 93  & 46  & 35  & 28 \\ \cline{2-11}
& BIT\textsuperscript{\textasteriskcentered}~\cite{gammell2020batch}  & 92  & 75  & 64  & 84  & 81  & 75  & 37  & 42  & 40 \\ \cline{2-11}
& RRT\textsuperscript{\textasteriskcentered}~\cite{karaman2011sampling}  & 9  & 12  & 17  & \textbf{100}  & \textbf{100}  & \textbf{100}  & 20  & 15  & 25 \\ \hline\hline
\multirow{14}{*}{\rotatebox[origin=c]{90}{Mean Planning Time (std, max) [ms]}}
& \multirow{2}{*}{\textbf{APP}}  & \textbf{.005}  & \textbf{.006}  & \textbf{.006}  & \textbf{.005}  & \textbf{.005}  & \textbf{.007}  & \textbf{.005}  & \textbf{.005}  & \textbf{.006} \\
&  & \scriptsize{(.001, .009)} & \scriptsize{(.001, .013)} & \scriptsize{(.001, .014)} & \scriptsize{(.001, .008)} & \scriptsize{(.001, .009)} & \scriptsize{(.001, .011)} & \scriptsize{(.001, .009)} & \scriptsize{(.001, .009)} & \scriptsize{(.001, .008)}\\ \cline{2-11}
& \multirow{2}{*}{Lightning~\cite{berenson2012robot}}  & 124  & 149  & 162  & 31.1  & 37.5  & 48.6  & 70.7  & 77.5  & 98.4 \\
&  & \scriptsize{(28.2, 221)} & \scriptsize{(63.3, 413)} & \scriptsize{(82.2, 552)} & \scriptsize{(13.7, 76.8)} & \scriptsize{(20.0, 96.1)} & \scriptsize{(30.2, 144)} & \scriptsize{(80.3, 542)} & \scriptsize{(63.0, 307)} & \scriptsize{(84.6, 414)}\\ \cline{2-11}
& \multirow{2}{*}{E-Graphs~\cite{Phillips-RSS-12}}  & 988  & 1044  & 904  & 203  & 273  & 275  & 162  & 210  & 176 \\
&  & \scriptsize{(587, 1943)} & \scriptsize{(484, 1952)} & \scriptsize{(364, 1520)} & \scriptsize{(92.0, 824)} & \scriptsize{(290, 1936)} & \scriptsize{(250, 1918)} & \scriptsize{(55.2, 534)} & \scriptsize{(78.1, 575)} & \scriptsize{(78.6, 567)}\\ \cline{2-11}
& \multirow{2}{*}{RRTConnect~\cite{kuffner2000rrt}}  & 168  & 167  & 260  & 62.0  & 59.5  & 72.9  & 112  & 110  & 104 \\
&  & \scriptsize{(192, 1880)} & \scriptsize{(71.7, 488)} & \scriptsize{(228, 1363)} & \scriptsize{(46.7, 364)} & \scriptsize{(29.2, 151)} & \scriptsize{(38.1, 194)} & \scriptsize{(75.0, 311)} & \scriptsize{(71.6, 272)} & \scriptsize{(72.4, 279)}\\ \cline{2-11}
& \multirow{2}{*}{LazyPRM~\cite{kavraki2000path}}  & 1014  & 1204  & 1196  & 338  & 480  & 442  & 219  & 219  & 224 \\
&  & \scriptsize{(388, 1921)} & \scriptsize{(410, 1954)} & \scriptsize{(516, 1967)} & \scriptsize{(265, 1454)} & \scriptsize{(402, 1716)} & \scriptsize{(405, 1919)} & \scriptsize{(65.8, 386)} & \scriptsize{(72.6, 478)} & \scriptsize{(93.1, 568)}\\ \cline{2-11}
& \multirow{2}{*}{BIT\textsuperscript{\textasteriskcentered}~\cite{gammell2020batch}}  & 727  & 904  & 833  & 327  & 298  & 442  & 251  & 262  & 252 \\
&  & \scriptsize{(256, 1997)} & \scriptsize{(380, 1950)} & \scriptsize{(363, 1923)} & \scriptsize{(357, 1963)} & \scriptsize{(305, 1555)} & \scriptsize{(417, 1728)} & \scriptsize{(102, 556)} & \scriptsize{(92.9, 540)} & \scriptsize{(109, 633)}\\ \cline{2-11}
& \multirow{2}{*}{RRT\textsuperscript{\textasteriskcentered}~\cite{karaman2011sampling}}  & 572  & 539  & 566  & 137  & 137  & 137  & 193  & 187  & 193 \\
&  & \scriptsize{(128, 893)} & \scriptsize{(54.6, 616)} & \scriptsize{(29.0, 617)} & \scriptsize{(31.9, 196)} & \scriptsize{(31.7, 199)} & \scriptsize{(32.4, 214)} & \scriptsize{(50.3, 289)} & \scriptsize{(36.7, 261)} & \scriptsize{(50.5, 290)}\\ \hline\hline
\multirow{6}{*}{\rotatebox[origin=c]{90}{APP}}
& Preprocessing & \multirow{2}{*}{8.9} & \multirow{2}{*}{48.2} & \multirow{2}{*}{49.9} & \multirow{2}{*}{17.6} & \multirow{2}{*}{31.2} & \multirow{2}{*}{74.6} & \multirow{2}{*}{17.5} & \multirow{2}{*}{30.4} & \multirow{2}{*}{98.8} \\
& Time [min] & & & & & & & & & \\ \cline{2-11}
& Paths per Goal & \multirow{2}{*}{2.0 ({0.1})} & \multirow{2}{*}{3.2 ({0.7})} & \multirow{2}{*}{4.1 ({0.9})} & \multirow{2}{*}{2.0 ({0.0})} & \multirow{2}{*}{3.1 ({0.5})} & \multirow{2}{*}{6.1 ({4.0})} & \multirow{2}{*}{2.0 ({0.1})} & \multirow{2}{*}{3.1 ({0.4})} & \multirow{2}{*}{5.8 ({7.7})} \\
& (mean, std)& & & & & & & & & \\ \cline{2-11}
& \multirow{2}{*}{Memory [Mb]} & \multirow{2}{*}{9.7} & \multirow{2}{*}{15.6} & \multirow{2}{*}{20.6} & \multirow{2}{*}{24.8} & \multirow{2}{*}{37.8} & \multirow{2}{*}{70.6} & \multirow{2}{*}{17.5} & \multirow{2}{*}{28.5} & \multirow{2}{*}{54.7} \\
& & & & & & & & & & \\ \hline
\end{tabular}
\end{center}
\caption{
The table shows the success rates and planning times of APP and the baselines, as well as the preprocessing statistics for three different domains, with 1, 2 and 3 movable obstacles. 100 random tests were run for each experiment with the timeout of 2s. The same timeout was used for APP during preprocessing.
}
\label{table:results}
\end{table*}

\begin{figure}
\centering
\includegraphics[width=0.48\textwidth]{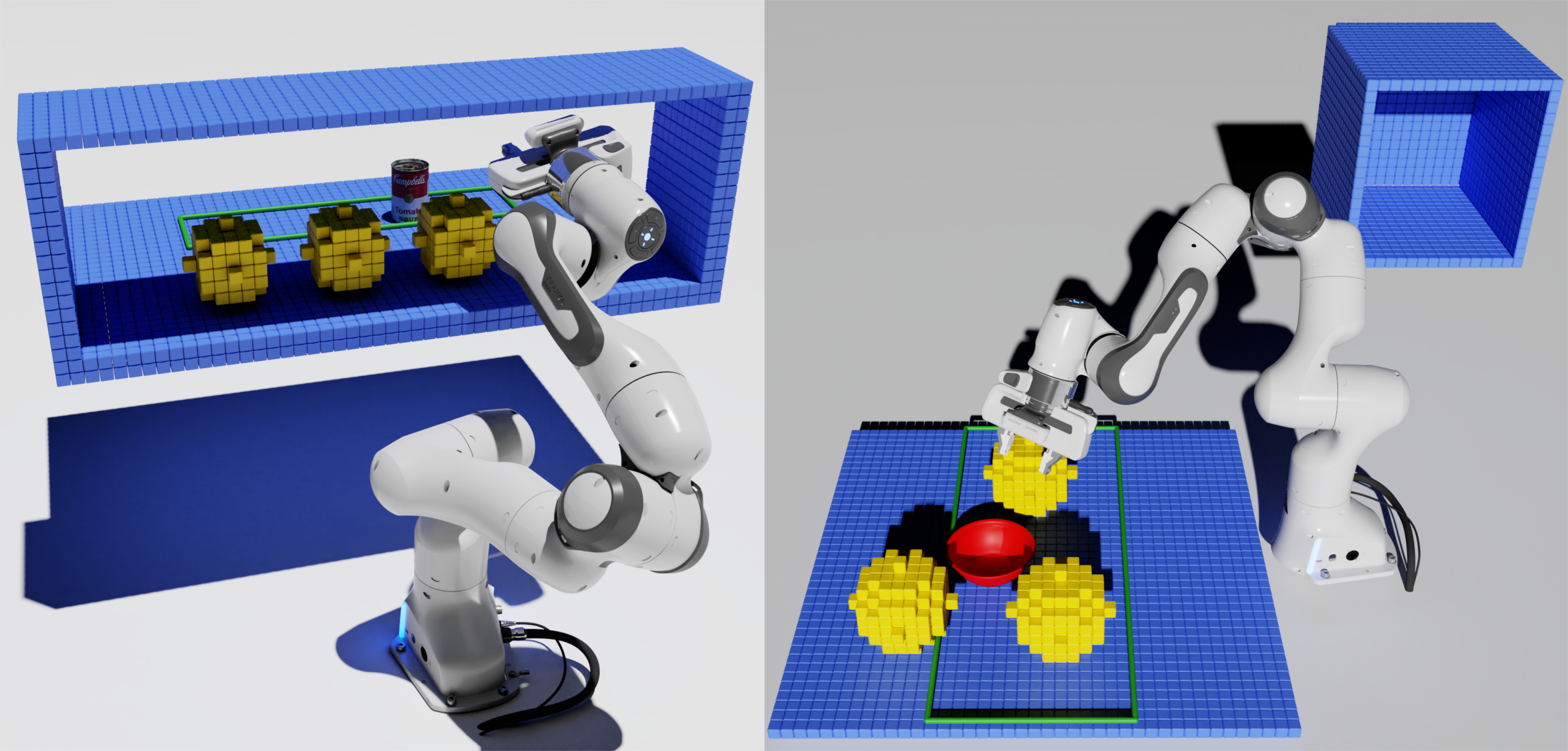}
\caption{Left: Shelving/unshelving task. Right: mail sorting task. Both tasks involve a mix of static obstacles and obstacles which may appear at many different positions. Target object is shown in red, and approximate movable obstacles are shown in yellow. The goal regions are depicted with green rectangles.}
\label{fig:scenarios}
\end{figure}

We evaluated APP on the 7 DoF Franka arm in three different semi-structured domains and compared its performance with other state-of-the-art sampling-based and search-based motion planning algorithms. For sampling-based algorithms we used OMPL's implementations~\cite{sucan2012open}

We consider the problem of reaching a feasible grasp pose $\phi \in \Phi$ for a target object $o_T$ at configuration $q(o_T)$, where $\Phi$ is a set of precomputed grasps (transformed into the robot's frame) for $o_T$. The resultant problem is a multi-goal planning problem where given $\Phi$, the robot has to find a collision free path $\pi$ to any $\phi \in \Phi$ while avoiding obstacles in \calW and meeting the success criterion. We use YCB objects~\cite{calli2015ycb} as targets with the grasps from~\cite{eppner2019billion}.

In all our experiments, we sort the grasp poses using a distance metric selecting the next best grasp being the one closest to the grasps already attempted so far. We use collision-aware IK to find valid grasps. For the baselines that support multi-goal planning (RRT*~\cite{karaman2011sampling}, BIT*~\cite{gammell2020batch}, LazyPRM~\cite{kavraki2000path}) in OMPL, the planner picks the top five grasps from the sorted list and plans to the corresponding goals within a single planning query. For the remaining baselines (Lightning~\cite{berenson2012robot}, RRT-Connect~\cite{kuffner2000rrt}, E-graphs~\cite{Phillips-RSS-12}) the planner sequentially iterates through all the grasps in $\Phi$ and terminates when the success criterion is met or until the timeout. APP uses Lightning as its underlying motion planner \calP.

We define the goal region \calG as a space of possible configurations of $o_T$ in the world. \calG is domain specific and depends on the allowable region of $o_T$ in the world as well as the geometry of $o_T$. To give an example, we used the bowl from the YCB dataset as the target object in our experiments. Since it is symmetric about the vertical axis and can rest on a planer surface, that limits the dimension of \calG to $\mathbb{R}^2$. For an object like a mug, its \calG would be in $SE(2)$. \calG is discretized to get a discrete set of poses $G$ of $o_T$.



\subsection{Experimental Scenarios}
Fig.~\ref{fig:scenarios} shows two of our example scenarios. In both the scenarios, we define the goal regions \calG as bounded $x,y$ planes since the target objects that we use are symmetric about the $z$-axis. The obstacle regions $\calQ(o_i)$ are also bounded $x,y$ planes and are identical for all obstacles in \calO.~\calG and $\calQ(o_i)$ both are discretized with a resolution of 2cm. We use occupancy grid to represent the occupancy of \calW which also has a 2cm discretization. The value of parameter $\epsilon$ (defined in Def.~\ref{def:envp}) in all the experiments is 20cm.
%
%
Table~\ref{table:results} shows the statistical results of all our experiments. The experiments were run on Intel Core i7-7800X CPU @3.5Ghz with 32 GB memory.

APP shows a success rate of 100\% for all the experiments, with several orders of magnitude speedup compared to the baselines. The baselines do well in the sorting domain but they suffer in shelving and the time-constrained sorting domains, which are more challenging.
Lightning's lower success rate for the shelving domain compared to APP is because randomly positioned obstacles can potentially create harder planning problems, compared to the problems it solves within the APP framework.
Lightning's lower success rate for the time-constrained domain is due to its higher planning time, which makes it harder to satisfy the overall timeout~$t_\textrm{task}$.
%
%
The performance of preprocessing and experience-based planners (LazyPRM, Lightning, E-graphs) drops with the increase in the number of movable obstacles.
In the time constrained domain, optimal sampling-based planners (RRT*, BIT*) do not perform better than non-optimal planners, because besides being slower, their convergence rate is not as fast as the time bound requires. Secondly more significant cost reduction is achieved by post-processing in our domains, which is used for all the planners.

\textbf{Shelving/Unshelving:}
 In the shelving/unshelving scenario, the robot is tasked with reaching a target object placed in the rear section of a shelf while avoiding other objects. 
 We approximate each $o_i \in \calO$ as a voxelized sphere of radius 6cm.
 This problem is challenging because the movable obstacles create narrow passages in the configuration space of the robot.
 Since the target object is symmetric about the vertical axis,~\calG is a bounded plane in the rear section of the shelve.
 
\textbf{Mail Sorting:}
The robot must pick up packages from a tabletop while avoiding collisions with other packages and put them in a cubby. We approximate each $o_i \in \calO$ as a voxelized sphere of radius 8cm. 
In addition to planning to a grasp pose, the planner needs to search for a valid pregrasp pose which gets harder with more clutter around the target object. Additionally, the target cubby creates a narrow passage for the motion planner.
An example of the preprocessing phase for a single start and goal pair for this domain is depicted in Fig.~\ref{fig:pr} 

\textbf{Time-Constrained Mail Sorting:}
We add an additional constraint on the robot that both the planning and execution must be completed within an overall timeout $t_\textrm{task}=2.0s$. Such constraints are often required for robots operating at conveyor belts~\cite{islam2020provably}. This setting makes the planning problem even harder because the manipulation planner not only has to plan fast but also must return a solution that is executable within the remainder of the time.

\subsection{Real-World Case Study}
We also tested APP on a pick-and-place task inspired by a kitchen environment, both in simulation and in the real world. We used Isaac Gym~\cite{liang2018gpu} for our simulation tests. The task is shown in Fig.~\ref{fig:real-world}. The goal is for the robot to pick up the red bowl, while avoiding two obstacles: large blue pitchers from the YCB object set~\cite{calli2015ycb}. We approximate the geometry of each pitcher with two spheres (one on top of the other).
The generation of paths with APP took less than 10 microseconds and was 100\% successful. Object poses were estimated by PoseCNN~\cite{xiang2017posecnn}.
On average, perception took $0.22 \pm 0.3$ seconds to return accurate poses; the high variance was due to some obstacle configurations being more challenging than others. For videos, see the supplementary materials.\footnote{Experiment videos: \url{https://bit.ly/34F8LrP}}

\begin{figure}[t]
    \centering
    \begin{subfigure}{0.235\textwidth}
        \includegraphics[trim={0.5cm 1.5cm 3cm 0},clip, width=\textwidth]{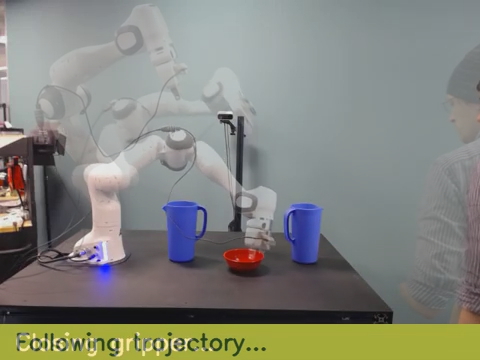}
        \caption{}
        \label{fig:real}
    \end{subfigure} 
    \begin{subfigure}{0.235\textwidth}
        \includegraphics[trim={8cm 0.55cm 8cm 0.55cm},clip, width=\textwidth]{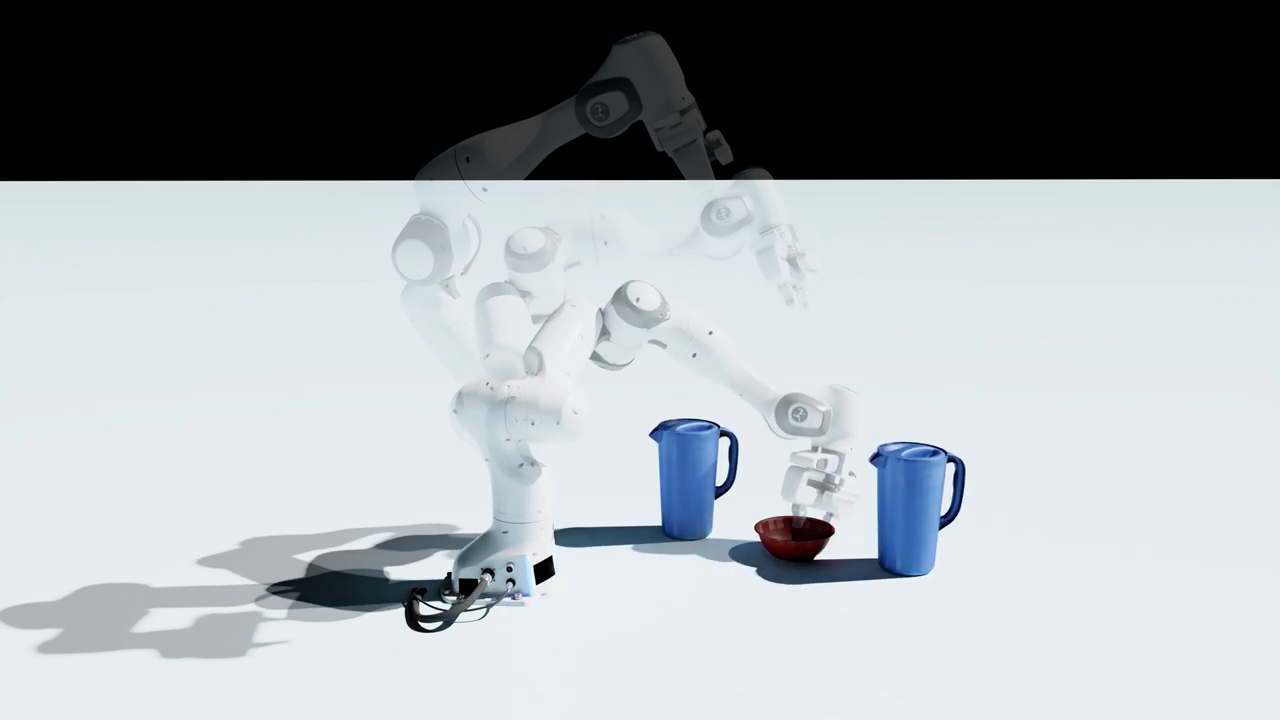}
        \caption{}
        \label{fig:sim}
    \end{subfigure}
    \caption{
    Real-world case study in a kitchen environment: 
     (\subref{fig:real})~Snapshots of a collision-free path genearated by APP in the real world.
    (\subref{fig:sim})~Snapshots of a collision-free path in simulation.
    }
\label{fig:real-world}
\end{figure}

\section{Conclusion and Future Work}
APP extends the applicability of existing fixed-time planning algorithms from highly-structured to semi-structured environments by explicitly accounting for movable obstacles during preprocessing.
A noteworthy limitation of APP is that its preprocessing times can grow drastically if the movable obstacles increase beyond a certain number.
An interesting future direction is to develop an ``anytime" variant of APP which tries to maximize the number of possible queries that it can handle at query time, within the allowed preprocessing time. Another direction is to relax our assumption on the obstacle geometry. Our real world experiment is a step towards this direction.

\bibliographystyle{IEEEtran}
\bibliography{main}

\end{document}